%% file: local_rademacher_networks-arxiv.tex
\documentclass[letterpaper,11pt]{article}
\usepackage[top=1in, bottom=1in, left=1in, right=1in]{geometry}

\usepackage{graphicx} % more modern
\usepackage[numbers]{natbib}
\usepackage{algorithm}
\usepackage{algorithmic}

\usepackage{amssymb}
\usepackage{amsthm}
\usepackage{amsmath}
\usepackage{multirow}
\usepackage{subfigure}
\usepackage{mathrsfs}
\usepackage{bbm}
\usepackage{bm}
\usepackage{eufrak}
\usepackage{mathtools}
\usepackage{slashbox}
\date{}

\graphicspath{{illustrations/}}

\input{symbol.tex}
\begin{document}
\title{An Empirical Study on Regularization of Deep Neural Networks by Local Rademacher Complexity}
\author{
Yingzhen Yang, Jiahui Yu, Xingjian Li, Jun Huan, Thomas S. Huang\\
Baidu Research, University of Illinois at Urbana-Champaign\\
\scriptsize \texttt{superyyzg@gmail.com, jyu79@illinois.edu, lixingjian@baidu.com, huanjun@baidu.com, huang@ifp.uiuc.edu}\\
}

\maketitle
\thispagestyle{empty}

\begin{abstract}
Regularization of Deep Neural Networks (DNNs) for the sake of improving their generalization capability is important and challenging. The development in this line benefits theoretical foundation of DNNs and promotes their usability in different areas of artificial intelligence. In this paper, we investigate the role of Rademacher complexity in improving generalization of DNNs and propose a novel regularizer rooted in Local Rademacher Complexity (LRC). While Rademacher complexity is well known as a distribution-free complexity measure of function class that help boost generalization of statistical learning methods, extensive study shows that LRC, its counterpart focusing on a restricted function class, leads to sharper convergence rates and potential better generalization given finite training sample. Our LRC based regularizer is developed by estimating the complexity of the function class centered at the minimizer of the empirical loss of DNNs. Experiments on various types of network architecture demonstrate the effectiveness of LRC regularization in improving generalization. Moreover, our method features the state-of-the-art result on the CIFAR-$10$ dataset with network architecture found by neural architecture search.
\end{abstract}
%under conditions weaker than Restricted Isometry Property widely used in compressive sensing literature
%\clearpage
%\setcounter{page}{1}
\section{Introduction}
Regularization on suitable function class is of great interest to statistical machine learning methods on various machine learning and pattern recognition problems, and it proves to improve generalization. Since the computation of the exact generalization error involves data distribution which is always unknown, most of the efforts are focused on generalization error bound. Such bound is typically comprised of two terms: an empirical loss evaluated on the finite training data and a regularization term that reflects the complexity of the function class that the classifier in question belongs to. Below is an example of error bound of the said form for kernel support vector machines \cite{Bartlett2003}.
\begin{MyExample}\label{example::generalization-bound}
Suppose the data and their corresponding labels $\{\bx_i, y_i\}_{i=1}^n$ are i.i.d. copies of a random couple $(X, Y)$ with unknown distribution, $k$ is a positive semi-definite (PSD) kernel. Then with probability at least $1-\delta$, the generalization error of the kernel support vector machines \cite{cristianini-taylor2000-kernel-svm}, where $f(x) = \sum\limits_i \alpha_i k(\bx_i, x)$,  has the following upper bound for binary classification problems:
\begin{small}\begin{align}\label{eq:generalization-bound-svm}
&\Pr[\by f(\bx) \le 0] \le \hat \E_n [\varphi(\by f(\bx))] + \frac{4B}{ n\gamma} \sqrt{\sum\limits_{i=1}^n k(\bx_i,\bx_i)}
+ \Big(\frac{8}{\gamma}\Big) \sqrt{\frac{\log 4/{\delta}}{2n}},
\end{align}\end{small}%
where $\varphi$ is an upper bound function for the $0-1$ classification error, $\gamma$ is a constant specifying the classification margin, and $\sum\limits_{ij} \alpha_i \alpha_j k(\bx_i,\bx_i) \le B$.
\end{MyExample}
Bound (\ref{eq:generalization-bound-svm}) is a classical example of generalization bounds derived using Rademacher complexity \cite{Bartlett2003,Koltchinskii01} for statistical learning methods. In Example~\ref{example::generalization-bound}, $B$ is in fact an upper bound for the Rademacher complexity of the function class that $f$ belongs to. According to (\ref{eq:generalization-bound-svm}), $B$ is a regularization term added to the empirical loss $\hat \E_n [\varphi(\by f(\bx))]$. The famous convex optimization problem of kernel support vector machines can be viewed as minimization of the objective $\E_n [\varphi(\by f(\bx))] + \lambda B$ in the hope that the resultant classifier enjoys controlled generalization error.

When machine learning proceeds to the stage of deep learning wherein Deep Neural Networks (DNNs) are widely used models \cite{YannLecunNature05-DeepLearning}, the effort in finding the aforementioned regularization scheme for DNNs is not very rewarding. In contrast, other regularization schemes such as dropout \cite{SrivastavaHKSS14-dropout}, batch normalization \cite{IoffeS15-bn} and mixup \cite{Zhang2017-mixup} work well in practice. While recent works have employed Rademacher complexity based regularizer to learn the rates of dropout, such as \cite{Zhai2018-rademacher-dropout}, the results are on relatively simple network architecture. Therefore, the prediction accuracy on standard dataset such as CIFAR-$10$ is not as good as well-known network architecture, e.g. \cite{HeZRS16-resnet}. Rademacher complexity has also been utilized to derive generalization bounds for DNNs through Lipschitz constant of the networks \cite{BartlettFT17-dnn-bound}. Bounding Lipschitz constant of DNNs is also shown to boost their robustness to adversarial examples \cite{CisseBGDU17-Lipschitz-dnn-adversarial}.

On the other hand, the bounds derived using conventional Rademacher complexity are perceived as ``loose'' ones. This is largely due to the fact that the Rademacher complexity is derived for the entire function class that the classification function belongs to, and larger function class leads to larger Rademacher complexity. The generalization error of the classifier obtained by minimizing such bound may still relatively large. This is further confirmed by \cite{Zhang2017-rethinking}, which points out that regularizer based on Rademacher complexity for DNNs may be trivial. Due to the strong fitting capability of DNNs, DNNs can fit uniformly $\pm 1$-valued Rademacher variables, leading to the (empirical) Rademacher complexity of value $1$. Therefore, any upper bound for such Rademacher complexity is trivial.

To alleviate the problem, the statistics literature has developed Local Rademacher Complexity (LRC) \cite{bartlett2005,koltchinskii2006}, wherein Rademacher complexity of restricted function class is derived to bound the generalization error of either an arbitrary function in the entire function class \cite{bartlett2005}, or the minimizer of the empirical loss \cite{koltchinskii2006}. Intuitively, local Rademacher complexity is smaller than its global counterpart which measures the complexity of the entire function class, and the resultant error bound is also sharper.

In this paper, we propose a regularizer based on local Rademacher complexity of a ball centered at the minimizer of the empirical loss. This is inspired by the observation that bounding the local Rademacher complexity of a ball centered at the minimizer of the empirical loss improves generalization \cite{lugosi2004-localized-random-penalties}. Note that the development of this regularizer does not void the claim made in \cite{Zhang2017-rethinking}, since the Rademacher complexity is evaluated on a local ball instead of the entire function class.

\section{Notation}
%In order to achieve fast convergence rate of Rademacher complexity, the statistics literature has exploited methods of estimating the complexity of a restricted or local function class. The restricted function can be a ball centered at the minimizer that minimizes the empirical loss, or centered at the zero function. For example, Bartlett et al. \cite{bartlett2005} obtains the generalization error bound for general function learning using local Rademacher complexity. More concretely, the bound for excess risk, $\E[\ell_f] - \E[\ell_{f^*}]$ where $\E[\ell_{f^*}] = \inf_{f \in \cF} \E[\ell_{f}]$ is bounded by the fixed point of the modulus of an upper bound for the local Rademacher complexity.
%The classifier $f$ corresponding to the network is $f(x) = \argmax_{y} h_y(\bx)$, i.e. the network takes the class with maximum output as the predicted class label for $\bx$.
Suppose the training data are $\cS = \{\bx_i,y_i\}_{i=1}^n$, $\cS$ are i.i.d. samples drawn from some unknown joint distribution $P_{XY}$ over the data $X \in \R^d$ and its class label $Y \in \left\{ {1,2,...,c} \right\}$. Let a deep neural network maps an input $\bx$ to a representation $h(\bx) \in \R^{c}$ upon which hinge loss or cross entropy loss is applied. Define the margin function $m_h(\bx,y) = h_{y}(\bx) - \max_{y' \ne y} h_{y'}(\bx)$, then the training instance $(\bx_i,y_i)$ is classified correctly if $m_h(\bx_i,y_i) \ge 0$. The hinge loss is ${\hat H_{n}(h)} = \frac{1}{n} \sum\limits_{i=1}^n {\Phi}(\frac{m_h(\bx_i,y_i)}{\gamma})$ where $\Phi$ is defined as
\begin{small}\begin{align}\label{eq:Phi}
\Phi(x) = \left\{
     \begin{array}{cl}
       1   &x < 0\\
       1-x &0 \le x \le 1 \\
       0   &x > 1. \\
     \end{array}
\right.
\end{align}\end{small}%

Similarly, let the softmax function in terms of the representation by $h$ be $\tilde h_y(\bx) = \frac{\exp{(h_y(\bx))}}{\sum\limits_{y'} \exp{(h_{y'}(\bx))}}$. The cross entropy loss function is defined as $\hat C_n (h) = \frac{1}{n} \sum\limits_{i=1}^n -\log {\tilde h_{y_i}(\bx_i)}$. The risk corresponding to cross entropy is then $C(f) = \E [-\log \tilde h_y(\bx)]$. Note that $C$ is the expectation of the negative logarithm of the softmax version of $h$.

\section{Regularization by Local Rademacher Complexity}

The Rademacher complexity \cite{Bartlett2003,Koltchinskii01} of a function class is defined below:
\begin{MyDefinition}\label{def:RC}
Let $\{\sigma_i\}_{i=1}^n$ be $n$ i.i.d. random variables such that $\Pr[\sigma_i = 1] = \Pr[\sigma_i = -1] = \frac{1}{2}$. The Rademacher complexity of a function class $\cH$ is defined as
\begin{small}\begin{align}\label{eq:RC}
&\cfrakR(\cH) = {\E}_{\{\sigma_i\},\{\bx_i\}}\left[\sup_{h \in \cH} {\frac{1}{n} \sum\limits_{i=1}^n {\sigma_i}{h(\bx_i)} } \right]
\end{align}\end{small}%
Its empirical version, i.e. the empirical Rademacher complexity, is defined as
\begin{small}\begin{align}\label{eq:empirical-RC}
&\hat \cfrakR(\cH) = {\E}_{\{\sigma_i\} }\left[\sup_{h \in \cH} {\frac{1}{n} \sum\limits_{i=1}^n {\sigma_i}{h(\bx_i)} } | \{\bx_i\} \right]
\end{align}\end{small}%
\end{MyDefinition}

Let $h_{\bw}(\bx)$ denote the feature mapping function of the neural network with explicit notation $\bw$ representing parameters of the neural network. The minimizer of the hinge loss is $\hat \bw = \argmin_{\bw} {\hat H_{n}(h)}$. The function class centered at the minimizer of the hinge loss is defined as $\cB_r = \{ {\Phi}(\frac{m_{h_{\bw}}(\bx, y)}{\gamma} \colon \bw \in B(\hat \bw, r) \}$ where $B(\bv, r)$ indicates an open ball centered at $\bv$ with radius $r > 0$. According to the definition of empirical Rademacher complexity (\ref{eq:empirical-RC}), the empirical Local Rademacher Complexity (LRC) of $\cB_r$ is

\begin{small}\begin{align}\label{eq:local-RC-margin}
&\hat \cfrakR^{(m)}(\hat \bw, r) = {\E}_{\{\sigma_i\}}\left[\sup_{ \bw \in B(\hat \bw, r) } {\frac{1}{n} \sum\limits_{i=1}^n {\sigma_i} {{\Phi}(\frac{m_{h_{\bw},y_i}(\bx_i)}{\gamma})} } | \{\bx_i,y_i\} \right]
\end{align}\end{small}%

We have the following theorem demonstrating the tight upper bound for $\hat \cfrakR^{(m)}(\hat \bw, r)$.

\begin{MyTheorem}\label{theorem::tight-bound-Rm}
Suppose that $h$ is a $L$-Lipschitz continuous function at $\hat \bw$, then
\begin{small}\begin{align}\label{eq:tight-bound-Rm}
&\hat \cfrakR^{(m)}(\hat \bw, r) \le \frac{1}{\gamma}\left|{\E}_{\{\sigma_i\}}\left[ \frac{1}{n} \sum\limits_{i=1}^n \sigma_i m_{h_{\hat \bw},y_i}(\bx_i)\right] \right| + \frac{3Lr}{\gamma}
\end{align}\end{small}%
\end{MyTheorem}
\begin{proof}
By the contraction property of Rademacher complexity \cite{LedouxTal91-probability},
\begin{small}\begin{align}\label{eq:tight-bound-Rm-seg1}
&\hat \cfrakR^{(m)}(\hat \bw, r) \le \frac{1}{\gamma} \left | {\E}_{\{\sigma_i\}}\left[\sup_{ \bw \in B(\hat \bw, r) } {\frac{1}{n} \sum\limits_{i=1}^n {\sigma_i} m_{h_{\bw},y_i}(\bx_i) } | \{\bx_i,y_i\} \right] \right| \nonumber \\
&\le \frac{1}{\gamma} \left|{\E}_{\{\sigma_i\}} \left[ \frac{1}{n} \sum\limits_{i=1}^n \sigma_i m_{h_{\hat \bw},y_i}(\bx_i)\right] \right| + \frac{1}{\gamma} \left | {\E}_{\{\sigma_i\}}\left[\sup_{ \bw \in B(\hat \bw, r) } \frac{1}{n} \sum\limits_{i=1}^n {\sigma_i} \big( m_{h_{\bw},y_i}(\bx_i) - m_{h_{\hat \bw},y_i}(\bx_i) \big) | \{\bx_i,y_i\} \right] \right| \nonumber \\
&\le \frac{1}{\gamma} \left|{\E}_{\{\sigma_i\}}\left[ \frac{1}{n} \sum\limits_{i=1}^n \sigma_i m_{h_{\hat \bw},y_i}(\bx_i)\right] \right| +  \frac{1}{\gamma} {\E}_{\{\sigma_i\}}\left[\sup_{ \bw \in B(\hat \bw, r) } \frac{1}{n} \sum\limits_{i=1}^n \left|{\sigma_i} \big( m_{h_{\bw},y_i}(\bx_i) - m_{h_{\hat \bw},y_i}(\bx_i) \big)\right| | \{\bx_i,y_i\} \right]  \nonumber \\
&\le \frac{1}{\gamma} \left|{\E}_{\{\sigma_i\}}\left[ \frac{1}{n} \sum\limits_{i=1}^n \sigma_i m_{h_{\hat \bw},y_i}(\bx_i)\right] \right| +  \frac{3Lr}{\gamma}.
\end{align}\end{small}%
The last inequality is due to the fact that $\left| m_{h_{\bw},y_i}(\bx_i) - m_{h_{\hat \bw},y_i}(\bx_i) \right| \le 3Lr$ when $\bw \in B(\hat \bw, r)$.
\end{proof}

Define $R(\bw) = |{\E}_{\{\sigma_i\}}\left[ \sigma_i m_{h_{\bw}}(\bx_i,y_i)\right]|$, then $R(\hat \bw)$ can serve as an approximate upper bound for $\hat \cfrakR^{(m)}(\hat \bw, r)$ when $r \to 0$. In order to bound the LRC $\hat \cfrakR^{(m)}(\hat \bw, r)$, we propose to minimize a regularized hinge loss, i.e. ${\hat H_{n}(h_{\bw})} + \lambda R(\bw)$, where $\lambda > 0$ is a weighting parameter. The intuition is that in practical scenarios, the regularized loss can achieve a very small value at the end of training, and the minimizer of the regularized loss can be a good approximation of the minimizer of ${\hat H_{n}(h_{\bw})}$ with bounded LRC of $\cB_r$ around it.

When it comes to the cross entropy loss, the corresponding LRC can be defined as follows in a manner similar to the case of hinge loss:

\begin{small}\begin{align}\label{eq:local-RC-cross-entropy}
&\hat \cfrakR^{c}(\hat \bw, r) = {\E}_{\{\sigma_i\}}\left[\sup_{ \bw \in B(\hat \bw, r) } {\frac{1}{n} \sum\limits_{i=1}^n {\sigma_i} \cdot {-\log \tilde h_{\bw, y_i}(\bx_i)} } | \{\bx_i,y_i\} \right]
\end{align}\end{small}%

The following theorem demonstrating the tight upper bound for $\hat \cfrakR^{(m)}(\hat \bw, r)$.
\begin{MyTheorem}\label{theorem::tight-bound-Rc}
Suppose that $h$ is a $L$-Lipschitz continuous function at $\hat \bw$, then
\begin{small}\begin{align}\label{eq:tight-bound-Rc}
&\hat \cfrakR^{c}(\hat \bw, r) \le \sqrt{2(c-1)} \left | {\E}_{\{\sigma_{ij}\}} \left[ \frac{1}{n}\sum\limits_{i=1}^n\sum\limits_{j \neq y_i} \sigma_{ij} (h_{\hat \bw,j}(\bx_i) - h_{\hat \bw_i,y_i}(\bx)) \right] \right|  + 2\sqrt{2(c-1)}{(c-1)Lr}
\end{align}\end{small}%
where ${\hat H_{n}(f)} = \frac{1}{n}  \sum\limits_{i=1}^n \Phi \Big(\frac{m_h(\bx_i,y_i)}{\gamma}\Big)$ is the empirical error of $f$ on the labeled data.
\end{MyTheorem}
\begin{proof}
Define function $r(\bv) = \log(1 + \sum\limits_{j=1}^m \exp(\bv_j) )$ for $\bv \in \R^m$. According to the mean value theorem, with $\bz = \tau \bx + (1-\tau)\by$ for some $\tau \in (0,1)$,
\begin{small}\begin{align}\label{eq:tight-bound-Rc-seg2}
&|r(\bx) - r(\by)| = |\nabla^{\top} r (\bz) (\bx - \by)| \le \|\nabla^{\top} r (\bz)\|_2 \|\bx-\by\|_2
\le \sqrt{m} \|\bx-\by\|_2.
\end{align}\end{small}%
Therefore, $r$ is a $\sqrt{m}$-Lipschitz continuous function. Let $\cT$ be a class of functions $t \colon \R^d \to \R^m$. Based on the vector-contraction inequality for Rademacher complexity \cite{Maurer2016-vector-rademacher}, we have
\begin{small}\begin{align}\label{eq:tight-bound-Rc-vector-rad}
{\E}_{\{\sigma_i\}} \left[ \sup_{t \in \cT} \sum\limits_{i=1}^n \sigma_i r(t(\bx_i)) \right] \le
\sqrt{2m} {\E}_{\{\sigma_{ij}\}} \left[ \sup_{t \in \cT} \sum\limits_{i=1}^n\sum\limits_{j=1}^m \sigma_{ij} t_j(\bx_i) \right],
\end{align}\end{small}%
where $t_j$ is the $j$-th component of $t$, $\{\sigma_{ij}\}$ are independent doubly indexed Rademacher variables.

On the other hand, the negative logarithm of softmax function appearing in (\ref{eq:local-RC-cross-entropy}) can be rewritten as
\begin{small}\begin{align*}%\label{eq:tight-bound-Rc-seg2}
&-\log \tilde h_{\bw, y} (\bx) =  \log \frac{\sum\limits_{y'} \exp{(h_{\bw, y'}(\bx))}}{\exp{(h_{\bw,y}(\bx))}}
= \log \big(1 + \sum\limits_{y': y' \neq y} \exp(h_{\bw,y'}(\bx) - h_{\bw,y}(\bx)) \big) \nonumber \\
\end{align*}\end{small}%
Applying (\ref{eq:tight-bound-Rc-vector-rad}) with $r = -\log \tilde h_{\bw, y}$ and $m=c-1$, by the definition of $\hat \cfrakR^{c}(\hat \bw, r)$ in (\ref{eq:local-RC-cross-entropy}), we have
\begin{small}\begin{align}\label{eq:tight-bound-Rc-seg3}
&\hat \cfrakR^{c}(\hat \bw, r) = {\E}_{\{\sigma_i\}}\left[\sup_{ \bw \in B(\hat \bw, r) } {\frac{1}{n} \sum\limits_{i=1}^n {\sigma_i} \cdot {-\log \tilde h_{\bw, y_i}(\bx_i)} } | \{\bx_i,y_i\} \right] \nonumber \\
&\le \sqrt{2(c-1)} {\E}_{\{\sigma_{ij}\}} \left[ \sup_{ \bw \in B(\hat \bw, r) } \frac{1}{n} \sum\limits_{i=1}^n\sum\limits_{j \neq y_i} \sigma_{ij} (h_{\bw,j}(\bx_i) - h_{\bw,y_i}(\bx)) \right].
\end{align}\end{small}%
Then (\ref{eq:tight-bound-Rc}) can be proved by (\ref{eq:tight-bound-Rc-seg3}) and argument similar to (\ref{eq:tight-bound-Rm-seg1}) in the proof of Theorem~\ref{theorem::tight-bound-Rc}.
\end{proof}

The tightness of (\ref{eq:tight-bound-Rm}) and (\ref{eq:tight-bound-Rc}) can be observed by letting $r \to 0$ and noting that $h$ is locally linear when the corresponding neural network uses ReLU as activation function. Algorithm~\ref{alg:LRC-Rm} and Algorithm~\ref{alg:LRC-Rc} describe the process of training neural networks with LRC regularization for hinge loss and cross entropy respectively. Note that the regularization term $R$ is computed according to the upper bound for LRC (\ref{eq:tight-bound-Rm}) and (\ref{eq:tight-bound-Rc}) with $r \to 0$. In addition, the computation of $R$ in both algorithms is simple and efficient without introducing noticeable computational burden.
\begin{algorithm}[!h]
\renewcommand{\algorithmicrequire}{\textbf{Input:}}
\renewcommand\algorithmicensure {\textbf{Output:} }
\small
\caption{Training Deep Neural Networks with Hinge Loss and Regularization by Local Rademacher Complexity}
\label{alg:LRC-Rm}
\begin{algorithmic}[1]
\REQUIRE ~~\\
The training data $\{\bx_i,y_i\}_{i=1}^{n}$, the neural network $h$, the regularization weight $\lambda$, $K \in \N$\\
\STATE \FOR{each epoch}
%\STATE $r=1$, initialize the sparse code matrix as ${\bZ}^{(0)} = \bZ_{\ell^{1}}$.
\FOR{minibatch $\cB = \{\bx_i, y_i\}_{i=1}^B$}
\STATE{$R = 0$.}
\FOR{$k \gets 1$ to $K$}
\STATE {Sample Rademacher variables $\{\sigma_i\}_{i=1}^B$.}
\STATE {$R \gets R + \frac{1}{B} \left| \sum\limits_{i=1}^B {\sigma_i} {m_h(\bx_i,y_i)} \right|$}
\ENDFOR
\STATE{$R \gets \frac{R}{K}$}
\STATE{Do one step gradient descent on $$L(\cB) = \frac{1}{n}  \sum\limits_{i=1}^B \Phi \Big(\frac{m_h(\bx_i,y_i)}{\gamma}\Big) + \lambda R$$}
%\IF{}
%\STATE{\textbf{break}}
%\ELSE
%\STATE{$m=m+1$.}
%\ENDIF
%and in each step of coordinate descent use the feature-sign search algorithm to solve the optimization problem (\ref{eq:objfunci}).
\ENDFOR
\ENDFOR
\ENSURE the trained neural network
\end{algorithmic}
\end{algorithm}

\begin{algorithm}[!htb]
\renewcommand{\algorithmicrequire}{\textbf{Input:}}
\renewcommand\algorithmicensure {\textbf{Output:} }
\small
\caption{Training Deep Neural Networks with Cross Entropy and Regularization by Local Rademacher Complexity}
\label{alg:LRC-Rc}
\begin{algorithmic}[1]
\REQUIRE ~~\\
The training data $\{\bx_i,y_i\}_{i=1}^{n}$, the neural network $h$, the regularization weight $\lambda$, $K \in \N$\\
\STATE \FOR{each epoch}
%\STATE $r=1$, initialize the sparse code matrix as ${\bZ}^{(0)} = \bZ_{\ell^{1}}$.
\FOR{minibatch $\cB = \{\bx_i, y_i\}_{i=1}^B$}
\STATE{$R = 0$.}
\FOR{$k \gets 1$ to $K$}
\STATE {Sample Rademacher variables $\{\sigma_{ij}\}_{1 \le i\le B, 1 \le j\le c-1}$.}
\STATE {$R \gets R + \frac{1}{Bc} \left|  \sum\limits_{i=1}^B\sum\limits_{j \neq y_i} \sigma_{ij} (h_{\bw,j}(\bx_i) - h_{\bw_i,y_i}(\bx)) \right|$}
\ENDFOR
\STATE{$R \gets \frac{R}{K}$}
\STATE{Do one step gradient descent on $$L(\cB) = \frac{1}{n} \sum\limits_{i=1}^n -\log {\tilde h_{y_i}(\bx_i)} + \lambda R$$}
%\IF{}
%\STATE{\textbf{break}}
%\ELSE
%\STATE{$m=m+1$.}
%\ENDIF
%and in each step of coordinate descent use the feature-sign search algorithm to solve the optimization problem (\ref{eq:objfunci}).
\ENDFOR
\ENDFOR
\ENSURE the trained neural network
\end{algorithmic}
\end{algorithm}

\section{Experimental Result}

\subsection{Experiment on ResNet}
%The improved generalization is evidenced by the reduced loss on the test data of the CIFAR-$10$ dataset when training ResNet-$18$.
We conduct experiments on the CIFAR-$10$ dataset \cite{Krizhevsky2009-cifar} in this subsection to demonstrate that the proposed regularization by LRC helps improve generalization of standard network architecture, i.e. Residual Network (ResNet) \cite{HeZRS16-resnet}. The CIFAR-$10$ dataset has $50000$ training images and $10000$ test images with $10$ classes. We hold out $5000$ images randomly chosen from the original training images as validation set, and the validation set is used to choose the regularization weight $\lambda$ from $\{0.1,0.5,1\}$.  We then train ResNet-$18$ with the chosen $\lambda$ on the original training set and evaluate the resultant model on the test set. The empirical loss is set to either hinge loss or cross entropy.

Both choices for empirical loss favor $\lambda=0.5$, and the corresponding test loss and test accuracy is shown in Table~\ref{table:validation-lambda-cifar}. The baseline neural network does not have LRC regularization, and other than that it has the same specification as the one with LRC regularization in every aspect. The average of test loss and test accuracy in the last $5$ epoches of the training process are reported. We can observe that the test loss is lower than that of the baseline. In addition, cross entropy enjoys more gain in accuracy by LRC regularization. The test loss and test error of ResNet-$18$ with respect to epoch number for cross entropy are shown in Figure~\ref{fig:lrc-cifar10-loss} and Figure~\ref{fig:lrc-cifar10-error} respectively. The training procedure is the same as that stated in \cite{HeZRS16-resnet}. $164$ epoches are used for training. The initial learning rate is $0.1$, and it is divide it by $10$ at $82$ and $123$ epoches respectively. Again, it is observed that the test loss of LRC regularization is consistently lower than that of the baseline, and its test error is also smaller accordingly.

%In order to investigate the behavior of the neural network with varying $\lambda$, we train ResNet-$18$ with cross entropy and different $\lambda$. It can observed that LRC regularization consistently achieves lower test loss across different $\lambda$, empirically indicating better generalization in terms of the risk (the expectation of loss). While LRC regularization does not guarantee improved accuracy for each $\lambda$, it is often better than the baseline.

\begin{table*}[!htb]
\centering
\scriptsize
\caption{\small Test loss and test accuracy on the CIFAR-$10$ dataset}
\begin{tabular}{|c|c|c|c|c|c|c|c|c|c|c|}
  \hline

   \backslashbox{Test}{Loss} &Hinge Loss   &Cross Entropy        \\\hline

   Test Loss: LRC(Baseline)       &0.128(0.153)  &0.123(0.161)         \\ \hline

   Test Accuracy: LRC(Baseline)    &94.16\%(94.13\%)  &94.34\%(93.95\%)        \\ \hline

\end{tabular}
\label{table:validation-lambda-cifar}
\end{table*}

%\begin{table*}[!htb]
%\centering
%\scriptsize
%\caption{\small Test Loss and Test Accuracy on CIFAR-$10$ with Different $\lambda$ for Hinge Loss}
%\begin{tabular}{|c|c|c|c|c|c|c|c|c|c|c|}
%  \hline
%
%
%   \backslashbox{Test}{$\lambda$} &$0.1$   &$0.3$  &$0.5$    &$0.7$    &$1$         \\\hline
%
%   Test Loss: LRC(baseline)       &0.131(0.156)  &0.131(0.152)    &0.128(0.152)      &0.123(0.160)     &0.141(0.161)        \\ \hline
%
%   Test Accuracy: LRC(baseline)    &94.36\%(94.01\%)  &94.11\%(94.44\%)  &94.25\%(94.22\%)  &94.43\%(94.05\%)  &93.58\%(94.11\%)        \\ \hline
%
%
%\end{tabular}
%\label{table:lambda-hinge-loss-cifar}
%\end{table*}

%\begin{table*}[!htb]
%\centering
%\scriptsize
%\caption{\small Test Loss and Test Accuracy on CIFAR-$10$ with Different $\lambda$ for Cross Entropy}
%\begin{tabular}{|c|c|c|c|c|c|c|c|c|c|c|}
%  \hline
%
%
%   \backslashbox{Test}{$\lambda$} &$0.1$   &$0.3$  &$0.5$    &$0.7$    &$1$         \\\hline
%
%   Test Loss: LRC(baseline)       &0.131(0.152)  &0.128(0.153)    &0.123(0.161)      &0.141(0.161)     &0.131(0.139)        \\ \hline
%
%   Test Accuracy: LRC(baseline)    &94.00\%(94.30\%)  &94.17\%(94.13\%)  &94.34\%(93.95\%)  &93.49\%(94.06\%)  &93.44\%(93.11\%)        \\ \hline
%
%
%\end{tabular}
%\label{table:lambda-cross-entropy-cifar}
%\end{table*}

\begin{minipage}{0.45\textwidth}
\begin{figure}[H]
\includegraphics[width=0.8\textwidth]{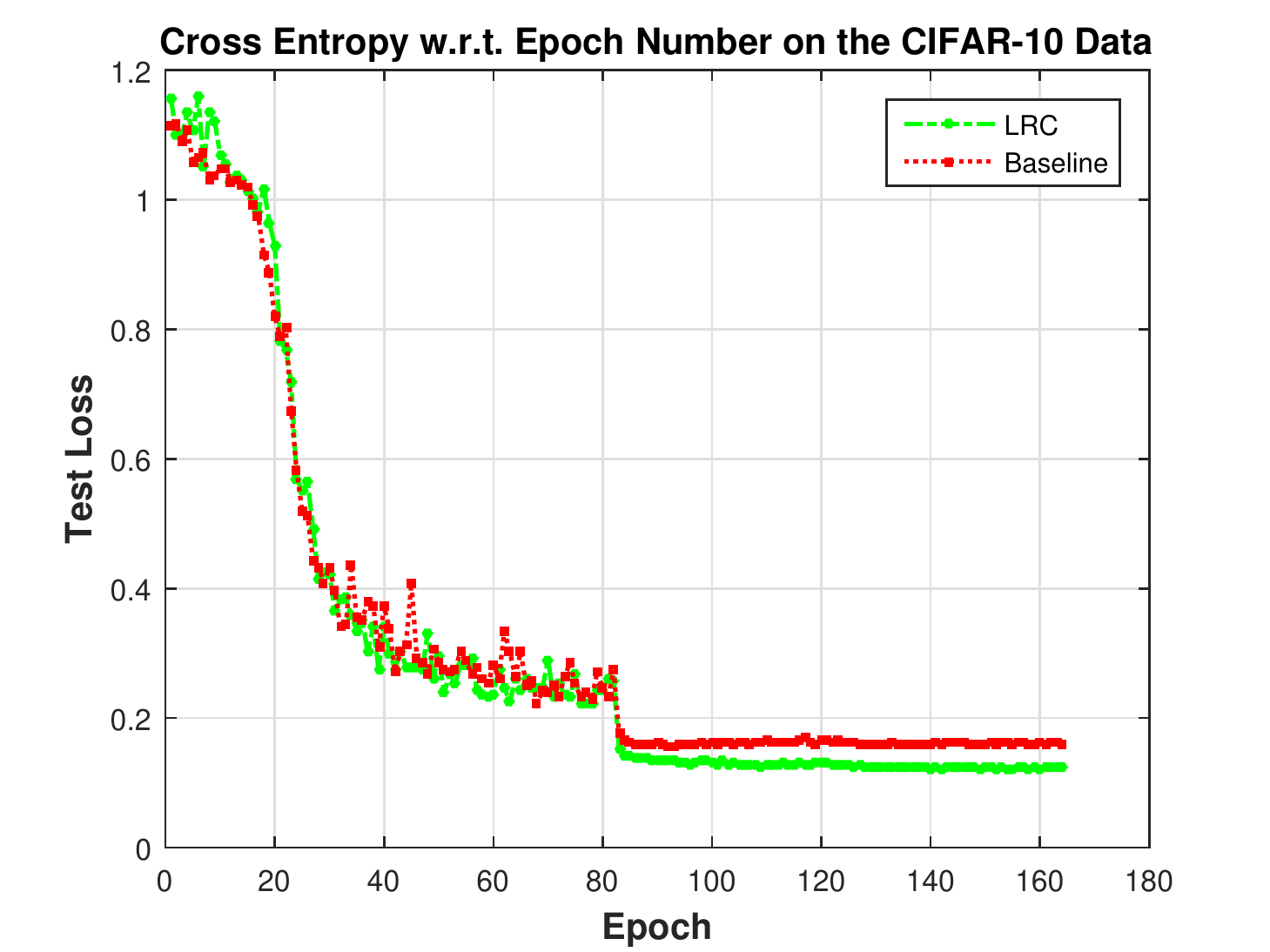}
\caption{\scriptsize Illustration of test loss on the CIFAR-$10$ dataset with $\lambda=0.5$}
\label{fig:lrc-cifar10-loss}
\end{figure}
\end{minipage}
\begin{minipage}{0.45\textwidth}
\begin{figure}[H]
\includegraphics[width=0.8\textwidth]{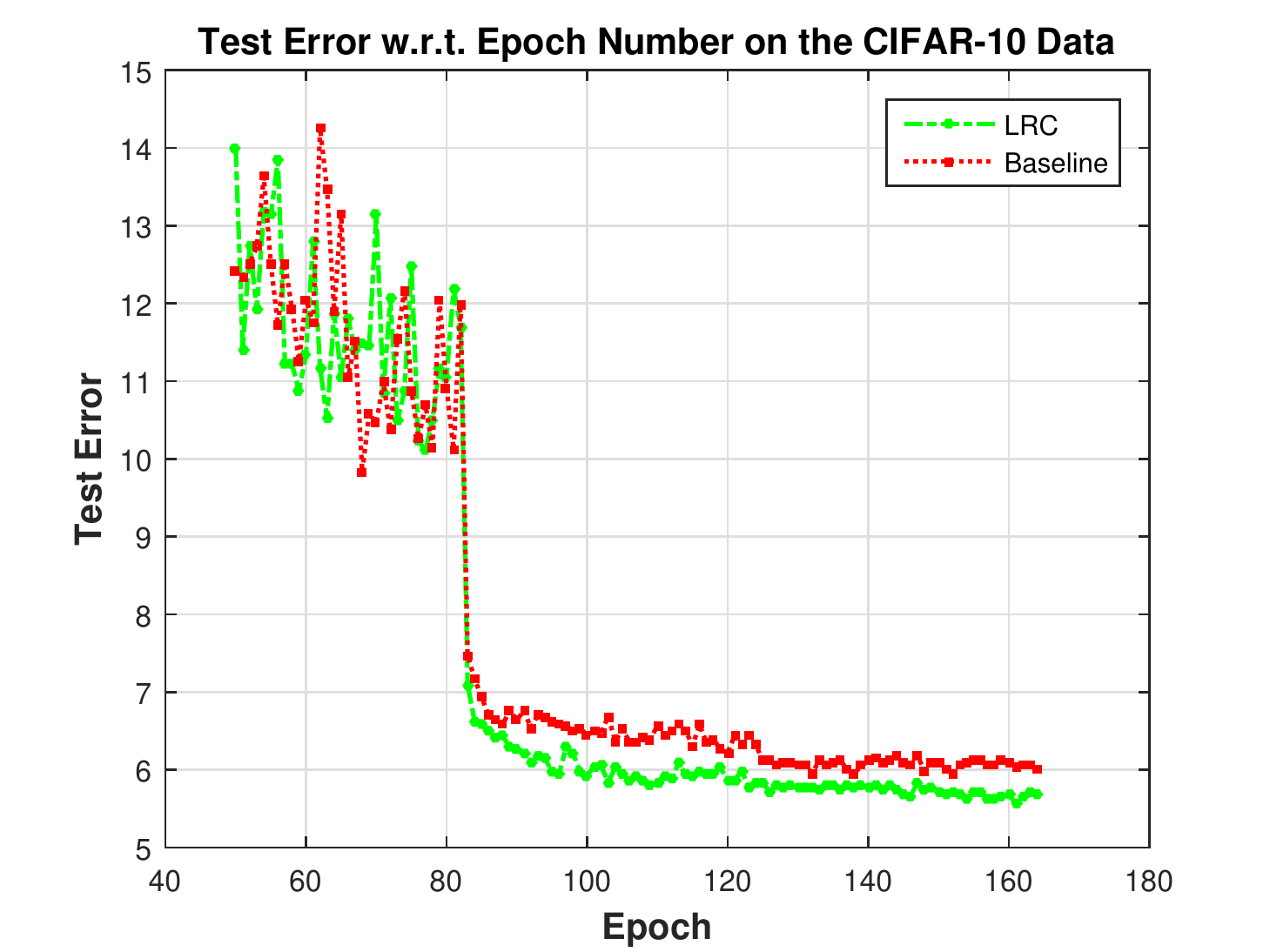}
\caption{\scriptsize Illustration of test error on the CIFAR-$10$ dataset with $\lambda=0.5$}
\label{fig:lrc-cifar10-error}
\end{figure}
\end{minipage}
%\end{minipage}

\begin{table*}[!htb]
\centering
\scriptsize
\caption{\small Classification accuracy on CIFAR-$10$ with different regularization weight $\lambda$}
\begin{tabular}{|c|c|c|c|c|c|c|c|}
  \hline

   \backslashbox{$\lambda$}{Model}       &$M_1$   &$M_2$     &$M_3$       &$M_4$      &$M_5$      &$M_6$   &$M_7$    \\\hline

   $\lambda=0$     &0.9689  &0.9711    &0.9667      &0.9686     &0.9689     &0.9678  &0.9742     \\ \hline

   $\lambda=0.1$   &0.9670  &0.9702    &0.9694      &0.9693     &0.9704     &0.9668  &0.9742     \\ \hline

   $\lambda=0.3$   &0.9688  &0.9719    &0.9688      &0.9708     &0.9701     &0.9679  &0.9742     \\ \hline

   $\lambda=0.5$   &0.9697  &0.9723    &0.9683      &0.9690     &0.9718     &0.9674  &0.9743     \\ \hline

   $\lambda=0.7$   &0.9706  &0.9710    &0.9672      &0.9697     &0.9702     &0.9675  &0.9744     \\ \hline

\end{tabular}
\label{table:lambda-DARTS-LRC}
\end{table*}
\begin{table*}[ht]
\centering
\scriptsize
\caption{\small Classification accuracy on CIFAR-$10$ with different regularization}
\begin{tabular}{|c|c|c|c|c|c|c|c|}
  \hline

   \backslashbox{Regularization}{Model}
                   &$M_1$   &$M_2$     &$M_3$       &$M_4$      &$M_5$      &$M_6$   &$M_7$    \\\hline

   None            &0.9689  &0.9711    &0.9667      &0.9686     &0.9689     &0.9678  &0.9742     \\ \hline

   mixup           &0.9741  &0.9733    &0.9740      &0.9740     &0.9742     &0.9713  &0.9789     \\ \hline

   LRC             &0.9691  &0.9719    &0.9691      &0.9690     &0.9707     &0.9680  &0.9744     \\ \hline

   mixup$+$LRC     &0.9712  &0.9734    &0.9731      &0.9752     &0.9730     &0.9732  &\textbf{0.9801}     \\ \hline

\end{tabular}
\label{table:lambda-mixup-LRC}
\end{table*}

\subsection{Experiment on Networks Obtained by Neural Architecture Search}
We evaluate the effect of LRC Regularization on more complex models found by the recent state-of-the-art neural architecture search algorithm, DARTS \cite{Liu2019-darts}, on the CIFAR-$10$ dataset. We obtain three types of neural architecture by running DARTS three times. The three types of architecture of normal cells and reduction cells are illustrated at Figure~\ref{fig:lrc-model-arch}. Each type of architecture is trained with five different initial learning rates, which are $0.01$, $0.015$, $0.02$, $0.025$ and $0.03$. The learning rates are gradually reduced to zero following a cosine schedule. We use SGD with momentum of $0.9$ to optimize the weights. The weight decay is set to $0.0002$. Each model is trained for $600$ epochs with a mini-batch size of $96$. We then randomly select six models from all the fifteen models, denoted by $\{M_i\}_{i=1}^6$. We further perform majority voting on the six models and denote the ensemble model by $M_7$. The effect of LRC regularization with different regularization weight $\lambda$ on the six models is shown in Table~\ref{table:lambda-DARTS-LRC}. Models with LRC regularization outperform the original DARTS models in most cases, and the performance of LRC regularization is not sensitive with respect to $\lambda$. The ensemble model $M_7$ leads to further improvement on the accuracy.

\begin{figure}[htbp]
\centering
\subfigure{
\begin{minipage}[t]{0.3\linewidth}
\centering
\includegraphics[width=2.5in]{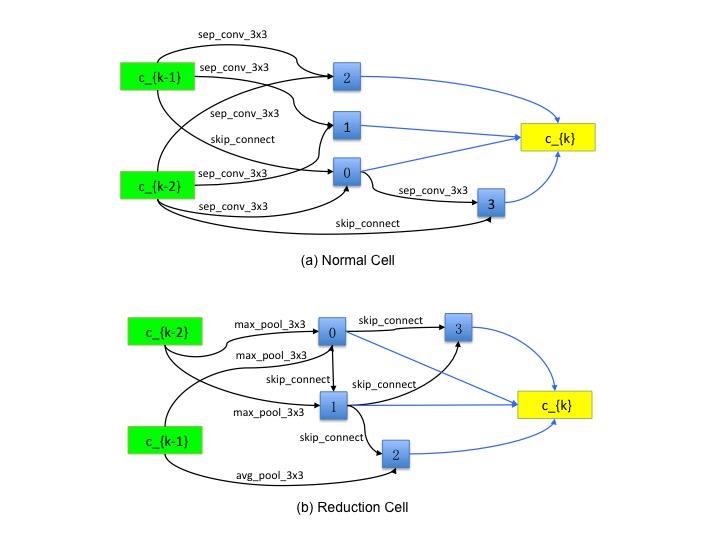}
%\caption{fig1}
\end{minipage}%
}%
\subfigure{
\begin{minipage}[t]{0.3\linewidth}
\centering
\includegraphics[width=2.5in]{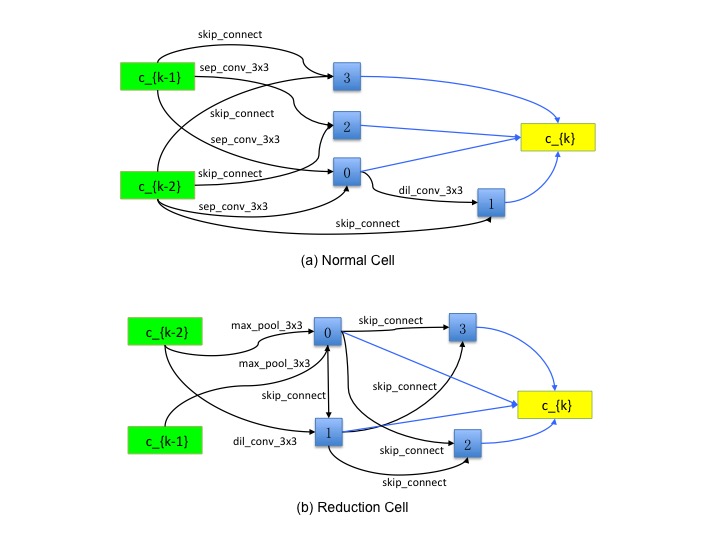}
%\caption{fig2}
\end{minipage}%
}%
\subfigure{
\begin{minipage}[t]{0.3\linewidth}
\centering
\includegraphics[width=2.5in]{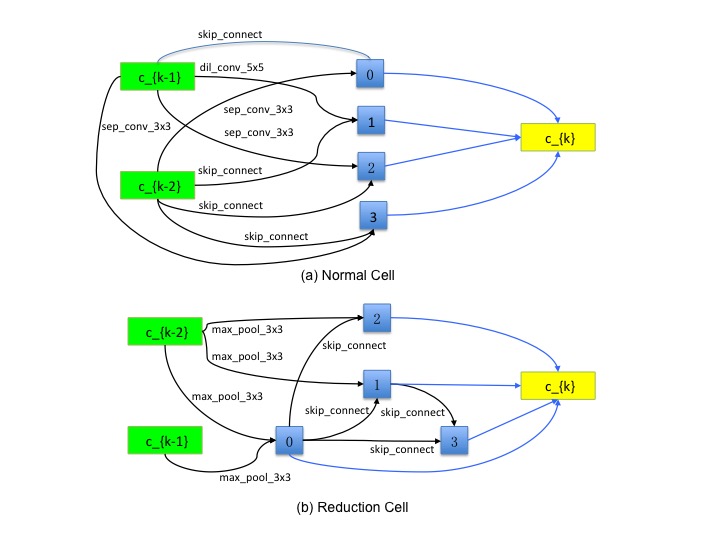}
%\caption{fig2}
\end{minipage}
}%
\centering
\caption{Three types of architecture found by DARTS}
\label{fig:lrc-model-arch}
\end{figure}

%We also analyzed the relationship between LRC regularization and mixup regularization on the same model, with $\lambda$ set to its default value $0.5$.
We also analyze the relationship between LRC regularization and mixup regularization on the same model. The best mixup coefficient $\alpha$ is selected for each model using cross-validation, then the mixup model is trained using all the training data. For the combination of mixup and LRC, we use the same $\alpha$ as that in the corresponding mixup experiment and a default value of $0.5$ for $\lambda$.

Table~\ref{table:lambda-mixup-LRC} demonstrates that the combination of LRC and mixup has performance comparable to that of only using mixup for $\{M_i\}_{i=1}^6$. However, the ensemble model $M_7$ considerably benefits from the combination of LRC and mixup. It should be emphasized that the ensemble model achieves the state-of-the-art accuracy on this dataset with the combination of LRC and mixup, $98.01\%$, to the best of our knowledge.

\section{Conclusion}
In this paper, we propose to improve the generalization capability of Deep Neural Networks (DNNs) by regularization through Local Rademacher Complexity (LRC). In contrast with its global counterpart, i.e. Rademacher complexity, LRC is estimated on a local ball centered at the minimizer of empirical loss. Therefore, the bound for LRC can be much smaller than that for Rademacher complexity and LRC has been proved to enjoy sharper generalization error bound by previous study. We develop regularization of DNNs by LRC for both hinge loss and cross entropy, and its effectiveness is demonstrated by empirical study with residual network and networks obtained by neural architecture search.

\section{Model and Software Release}
The related models using both PaddlePaddle and PyTorch frameworks are available at \url{https://paddlemodels.bj.bcebos.com/autodl/fluid_rademacher.tar.gz}. The open source PaddlePaddle code could be downloaded from \url{https://github.com/PaddlePaddle/AutoDL/tree/master/LRC}.

\section{Acknowledgement}
We would like to express our gratitude to Guanzhong Wang at Baidu Inc. for his great efforts in implementing our algorithm by PaddlePaddle and releasing the PaddlePaddle code and models based on the LRC regularization described in this paper.

\bibliography{mybib}
\bibliographystyle{unsrt}

%\input{supplementary}
%\appendix

\end{document}

%% file: symbol.tex
\newcommand{\cB}{\mathcal{B}}

\newcommand{\cH}{\mathcal{H}}

\newcommand{\cfrakR}{\mathfrak{R}} %\usepackage{eufrak}

\newcommand{\cS}{\mathcal{S}}
\newcommand{\cT}{\mathcal{T}}

\newcommand{\bv}{\mathbf{v}}
\newcommand{\bx}{\mathbf{x}}
\newcommand{\by}{\mathbf{y}}
\newcommand{\bz}{\mathbf{z}}

\newcommand{\bw}{\mathbf{w}}

\newcommand{\N}{{\rm I}\kern-0.18em{\rm N}}

\newcommand{\R}{{\rm I}\kern-0.18em{\rm R}}
\newcommand{\h}{{\rm I}\kern-0.18em{\rm H}}
\newcommand{\K}{{\rm I}\kern-0.18em{\rm K}}
\newcommand{\p}{{\rm I}\kern-0.18em{\rm P}}
\newcommand{\E}{{\rm I}\kern-0.18em{\rm E}}
\newcommand{\Z}{{\rm Z}\kern-0.18em{\rm Z}}

\newcommand{\1}{{\rm 1}\kern-0.25em{\rm I}}

\newcommand{\pn}{\p_{\kern-0.25em n}}
\newcommand{\pnm}{\p_{\kern-0.25em n,m}}
\newcommand{\psubm}{\p_{\kern-0.25em m}}

\newcommand{\BigO}[1]{{\operatorname{O}}}

\DeclareMathOperator*{\argmin}{arg\,min}

\newtheorem{MyDefinition}{Definition}

\newtheorem{MyTheorem}{Theorem}

\newtheorem{MyExample}{Example} 